\documentclass[runningheads]{llncs}
\usepackage[T1]{fontenc}
\usepackage{graphicx,verbatim,tabularx}
\usepackage{amsmath,algorithm,algorithmic,multicol,amssymb,hyperref,multirow,array}
\usepackage{marvosym}

\usepackage{color}

\begin{document}
\title{Paired Image Generation with Diffusion-Guided Diffusion Models}

\author{
Haoxuan Zhang\inst{1,2\star}\and
Wenju Cui\inst{1,2\star}\and
Yuzhu Cao\inst{1,2}\and
Tao Tan\inst{3} \and
Jie Liu\inst{4} \and
Yunsong Peng\inst{1,2,5}\textsuperscript{\Letter} \and
Jian Zheng\inst{1,2}\textsuperscript{\Letter}
}
\authorrunning{H. Zhang et al.}
\institute{
School of Biomedical Engineering (Suzhou), Division of Life Sciences and Medicine, University of Science and Technology of China, Hefei Anhui, China\\ \and
Suzhou Institute of Biomedical Engineering and Technology, Chinese Academy of Science, Suzhou Jiangsu, China\\ \and
Faculty of Applied Sciences, Macao Polytechnic University, Macao, China \and
Department of Radiology, Suzhou Municipal Hospital, Suzhou Jiangsu, China \and
Key Laboratory of Advanced Medical Imaging and Intelligent Computing of Guizhou Province, Guizhou Provincial People's Hospital, Guiyang Guizhou, China
}
\maketitle
{\let\thefootnote\relax\footnotetext{$\star$: Equal contribution; \Letter: pys@mail.ustc.edu.cn, zhengj@sibet.ac.cn}} 
\begin{abstract}

The segmentation of mass lesions in digital breast tomosynthesis (DBT) images is very significant for the early screening of breast cancer. However, the high-density breast tissue often leads to high concealment of the mass lesions, which makes manual annotation difficult and time-consuming. As a result, there is a lack of annotated data for model training. Diffusion models are commonly used for data augmentation, but the existing methods face two challenges. First, due to the high concealment of lesions, it is difficult for the model to learn the features of the lesion area. This leads to the low generation quality of the lesion areas, thus limiting the quality of the generated images. Second, existing methods can only generate images and cannot generate corresponding annotations, which restricts the usability of the generated images in supervised training. In this work, we propose a paired image generation method. The method does not require external conditions and can achieve the generation of paired images by training an extra diffusion guider for the conditional diffusion model. 
During the experimental phase, we generated paired DBT slices and mass lesion masks. Then, we incorporated them into the supervised training process of the mass lesion segmentation task. The experimental results show that our method can improve the generation quality without external conditions. Moreover, it contributes to alleviating the shortage of annotated data, thus enhancing the performance of downstream tasks. The source code is available at \url{https://github.com/zhanghx1320/PIG}.



\end{abstract}
\section{Introduction}
The segmentation of masses in Digital Breast Tomosynthesis (DBT) images is of great significance for the early screening of breast cancer~\cite{concealment}. However, the dense fibroglandular tissue structure in mammograms often results in highly concealed mass lesions, as illustrated in Fig.~\ref{fig0}~\cite{concealment1}. This inherent concealment poses significant challenges for radiologists in manual annotation, leading to labor-intensive workflows and a severe scarcity of high-quality annotated datasets for training deep learning models.
To mitigate this data limitation, diffusion models have emerged as a promising solution for synthetic data augmentation in medical imaging~\cite{kazerouni2023102846,minim,maisi}. The generative framework of diffusion models can be broadly categorized into two paradigms. 
Unconditional generation methods, such as DDPM~\cite{ddpm} and DDIM~\cite{ddim}, synthesize realistic images from pure noise by progressively denoising random Gaussian distributions through Markov chains.
Conditional generation methods, on the other hand, achieve controllable generation by adding guidance signals to the denoising process of the network. These guidance signals can be in the form of classes~\cite{classifier_guidance,classifier_free}, texts~\cite{ldm,glide}, images~\cite{controlnet,sdg,sgd}, or features~\cite{Self_Attention_Guidance}. 

\begin{figure}
\centering
\includegraphics[width=0.8\textwidth]{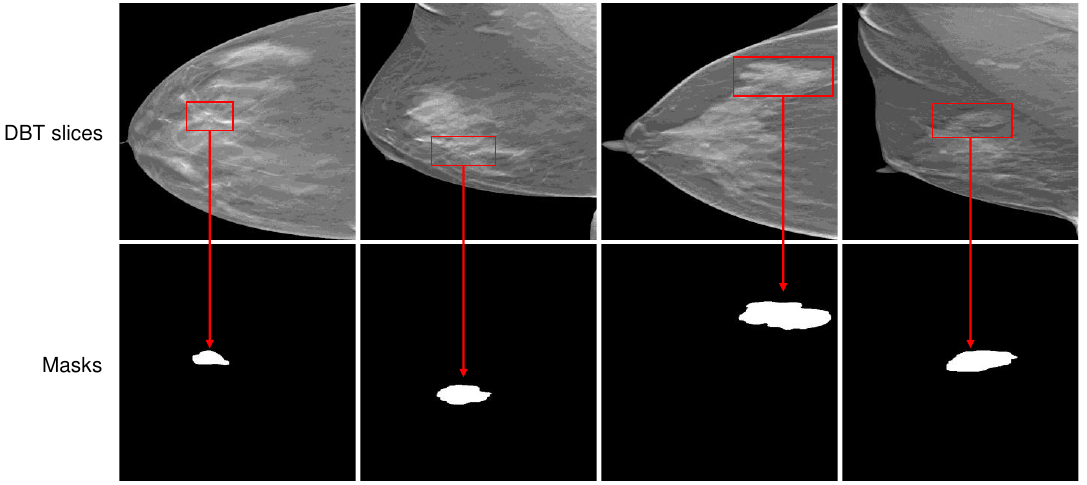}
\caption{Hidden mass lesions in DBT slices.} \label{fig0}
\end{figure}

However, the application of the above methods in the medical field has two problems. 
First, for unconditional generation methods, due to the concealment of the lesions, it is difficult for the model to learn the features of the lesion areas. This results in a relatively low generation quality of the lesion areas, thereby restricting the overall quality of the generated images~\cite{sgd}. 
Second, existing methods only generate images and cannot generate the corresponding annotations. Therefore, conditional generation methods are often used to obtain annotated images, where the annotations are input to guide the generation process. However, the conditions rely on manual input, which limits the diversity of annotations. As a result, the diversity of the generated images is restricted, and the generated images cannot efficiently improve the performance of the model through supervised training in downstream tasks.

To tackle these two problems, we propose a Paired Image Generation (PIG) method based on diffusion models. This method does not require external conditional inputs. Instead, it leverages the paired relationship between images as the guiding signal within the generation process. Specifically, we model the generation process of paired images first. Through mathematical derivations, we prove that the unconditional generation process of paired images can be equivalent to two diffusion processes guided mutually. Thus, similar to the classifier-guided diffusion model approach~\cite{classifier_guidance}, we can achieve the generation of paired images by training an additional diffusion guider for the conditional diffusion model. In the experimental stage, we generated paired DBT slices and mass lesion masks, and added the images generated by different methods to the training of the mass segmentation model. The experimental results show that compared with other generation methods, the images generated by our method have a 15.66 improvement in the FID metric and a 2.24\% improvement in the Dice metric. This indicates that our method contributes to enhancing the generation quality and alleviating the shortage of annotated data.

\section{Methodology}
\subsection{Revisiting Diffusion Models}
We adopt the diffusion models as the backbone of our proposed PIG method. The diffusion process of the diffusion models consists of the forward process and the reverse process~\cite{ddpm,ddim,sgd}. 
\subsubsection{The Forward Process.}
In the forward process, Gaussian noise is gradually added to the image until the image becomes pure noise.  The forward process of the diffusion models is Markovian, and the probability distribution $q$ of this process satisfies Eq. (\ref{e1}), where $x_0$ is the clean image, $x_t$ is the noised image at the $t$-th time-step, $T$ is the number of time-steps, $\mathcal{N}$ denotes the Gaussian distribution, $\beta_t\in(0,1)$ is the predefined noise schedule and $I$ is an identity matrix of the same shape as $x_0$. Eq. (\ref{e3}) represents the single-step noise addition process, and the $t$-step noise addition process can be obtained through recursion as shown in Eq. (\ref{e4}), where $\bar\alpha_t=\prod_{i=1}^t(1-\beta_i)\rightarrow 0$. 
\begin{equation}\label{e1}
    q(x_{1:T}|x_0)=\prod_{t=1}^T q(x_t|x_{t-1}), q(x_t|x_{t-1})=\mathcal{N}(x_t;\sqrt{1-\beta_t}x_{t-1},\beta_tI)
\end{equation}
\begin{equation}\label{e3}
    x_t=\sqrt{1-\beta_t}x_{t-1}+\sqrt{\beta_t}\epsilon_t,\epsilon_t\sim\mathcal{N}(0,I)
\end{equation}
\begin{equation}\label{e4}
    x_t=\sqrt{\bar\alpha_t}x_0+\sqrt{1-\bar\alpha_t}\epsilon,\epsilon\sim\mathcal{N}(0,I)
\end{equation}

\subsubsection{The Reverse Process.}
The reverse process generates images from pure noise by gradually denoising through a trained model $\theta$. Due to the concealment of lesions, in order to improve the generation quality of the lesion part, the mask of the lesion is often input as an external condition $c$, thereby guiding the model to pay more attention to the generation of the lesion area. The probability distribution $p_\theta$ of this process satisfies Eq. (\ref{e5}) and Eq. (\ref{e6}), where $p(x_T)=\mathcal N(x_T;0,I)$, $\mu_\theta(x_t,c,t)=\sqrt{\bar\alpha_{t-1}}x_0+\sqrt{1-\bar\alpha_{t-1}-\sigma_t^2}\cdot\epsilon_\theta(x_t,c,t)$ and $\sigma_t$ is the predefined noise variance. Specifically, the model predicts the noise $\epsilon_\theta(x_t,c,t)$ and obtains the predicted clean image $\hat x_0$ at the $t$-th time-step using Eq. (\ref{e7}). Subsequently, $x_{t-1}$ is obtained through Eq. (\ref{e8}) for the next denoising step.
\begin{equation}\label{e5}
    p_\theta(x_{0:T}|c)=p(x_T)\prod_{t=1}^Tp_\theta(x_{t-1}|x_t,c)
\end{equation}
\begin{equation}\label{e6}
    p_\theta(x_{t-1}|x_t,c)=\mathcal{N}(x_{t-1};\mu_\theta(x_t,c,t),\sigma_t^2I)
\end{equation}
\begin{equation}\label{e7}
    \hat x_0=\frac{x_t-\sqrt{1-\bar\alpha_t}\epsilon_\theta(x_t,c,t)}{\sqrt{\bar\alpha_t}}
\end{equation}
\begin{equation}\label{e8}
    x_{t-1}=\sqrt{\bar\alpha_{t-1}}\hat x_0+\sqrt{1-\bar\alpha_{t-1}-\sigma_t^2}\cdot\epsilon_\theta(x_t,c,t)+\sigma_t\epsilon_t,\epsilon_t\sim\mathcal{N}(0,I)
\end{equation}

\subsubsection{Limitations.} Since the condition $c$ relies on external input, manually creating these image conditions is time-consuming, and their diversity is limited. Therefore, we aim to generate $c$ simultaneously when generating images to obtain more diverse lesion annotations. In addition, another advantage of this approach is that it can introduce guiding signals into the image generation process without external condition input, thereby enhancing the generation quality. To achieve this goal, we propose a Paired Image Generation (PIG) method.
\begin{figure}
\centering
\includegraphics[width=0.8\textwidth]{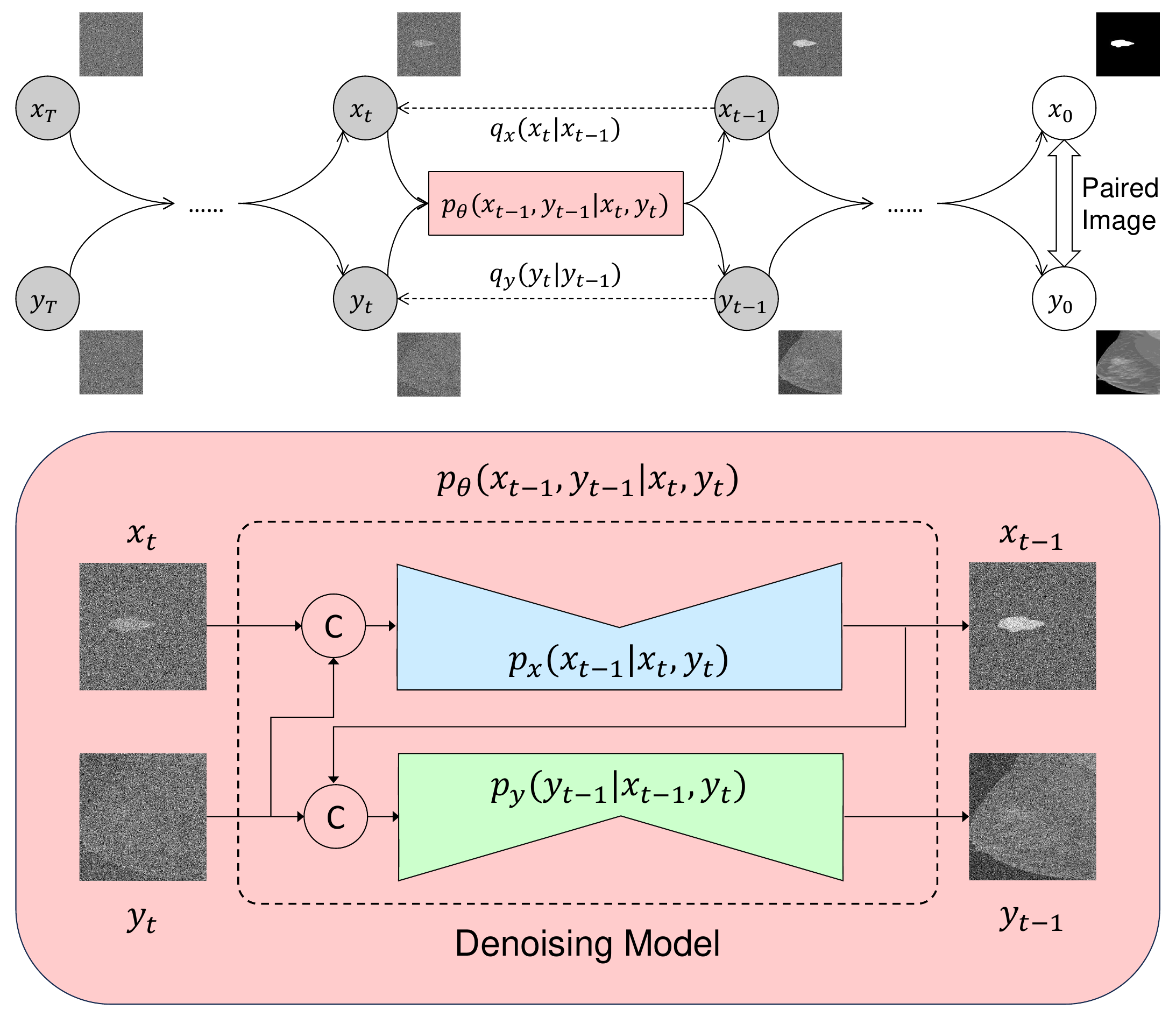}
\caption{Graphical model for paired image generation process.} \label{fig1}
\end{figure}

\subsection{Paired Image Generation}
In this section, we conduct the mathematical modeling for the PIG method based on the diffusion models. The PIG method combines two diffusion processes, and its graphical model is illustrated in Fig.~\ref{fig1}.

\subsubsection{The Forward Process.}
During the forward process, we adopt two Markovian noise addition processes similar to Eq. (\ref{e1})-(\ref{e4}) as shown in Eq. (\ref{e9})-(\ref{e11}), where $(x_0,y_0)$ represents the clean paired images and $\epsilon_1,\epsilon_2\sim\mathcal{N}(0,I)$.
\begin{equation}\label{e9}
    q(x_{1:T},y_{1:T}|x_0,y_0)=\prod_{t=1}^Tq_x(x_t|x_{t-1})q_y(y_t|y_{t-1})
\end{equation}
\begin{equation}\label{e10}
    q_x(x_t|x_{t-1})=\mathcal{N}(x_t;\sqrt{1-\beta_t}x_{t-1},\beta_tI),x_t=\sqrt{\bar\alpha_t}x_0+\sqrt{1-\bar\alpha_t}\epsilon_1
\end{equation}
\begin{equation}\label{e11}
    q_y(y_t|y_{t-1})=\mathcal{N}(y_t;\sqrt{1-\beta_t}y_{t-1},\beta_tI),y_t=\sqrt{\bar\alpha_t}y_0+\sqrt{1-\bar\alpha_t}\epsilon_2
\end{equation}

\subsubsection{The Reverse Process.}
The probability distribution $p_\theta$ of the joint denoising process satisfies Eq. (\ref{e12}), which can be implemented with two diffusion models $x$ and $y$ to fit $p_x$ and $p_y$ respectively as proved in Eq. (\ref{e13}). In Eq. (\ref{e12})-(\ref{e13}), $P(\cdot)$ is the abbreviation of probability distribution, $P(x_T)=\mathcal N(x_T;0,I)$ and $P(y_T)=\mathcal N(y_T;0,I)$. 
Specifically, we can train an additional diffusion guider as model $x$. It takes $y_t$ as the guiding signal to generate $x_{t-1}$ for guiding model $y$. Moreover, according to Eq. (\ref{e7})-(\ref{e8}), in the generation process of $x_{t-1}$, $\hat x_0$ will be predicted first. Therefore, the conditional diffusion model $p_y(y_{t-1}|x_0,y_t)$ (such as Maisi~\cite{maisi}, a generative foundation model pretrained on large-scale unannotated datasets) can be used as model $y$ to avoid the influence of the known noise $\epsilon_x$. During the generation process, this diffusion-guided diffusion model architecture can generate guiding signals for each other, thereby improving the image quality. The training and sampling processes are shown in Alg \ref{alg1}-\ref{alg2}, where we set $\sigma_t=0$.
\begin{equation}\label{e12}
    p_\theta(x_{0:T},y_{0:T})=P(x_T)P(y_T)\prod_{t=1}^Tp_\theta(x_{t-1},y_{t-1}|x_t,y_t)
\end{equation}

\begin{proposition}
The joint denoising process $p_\theta$ of the paired images can be implemented with two conditional denoising processes $p_x$ and $p_y$.
\end{proposition}
\begin{proof}According to the conditional probability formula and the Bayesian probability formula, the derivation process is as follows.
    \begin{align}\label{e13}
        &p_\theta(x_{t-1},y_{t-1}|x_t,y_t)\nonumber\\
        =&p_x(x_{t-1}|x_t,y_t)\cdot P(y_{t-1}|x_{t-1},x_t,y_t)\nonumber\\
        =&p_x(x_{t-1}|x_t,y_t)\cdot \frac{p_y(y_{t-1}|x_{t-1},y_t)\cdot P(x_t|x_{t-1},y_{t-1},y_t)}{P(x_t|x_{t-1},y_t)}\nonumber\\
        =&p_x(x_{t-1}|x_t,y_t)\cdot \frac{p_y(y_{t-1}|x_{t-1},y_t)\cdot q_x(x_t|x_{t-1})}{q_x(x_t|x_{t-1})}\nonumber\\
        =&p_x(x_{t-1}|x_t,y_t)\cdot p_y(y_{t-1}|x_{t-1},y_t)
    \end{align}
\end{proof}

\begin{algorithm}[h]
\caption{The training process of the diffusion guider.}
\label{alg1}
\renewcommand{\algorithmicrequire}{\textbf{Input:}}
\renewcommand{\algorithmicensure}{\textbf{Output:}}
\begin{algorithmic}[1]
\REQUIRE dataset distribution $p(x_0,y_0)$, noise schedule $\beta_t$, number of timesteps $T$
\ENSURE trained $model_x$
\REPEAT
    \STATE sample $(x_0, y_0)\sim p(x_0,y_0), \epsilon_1\sim \mathcal{N}(0,I),\epsilon_2 \sim \mathcal{N}(0,I), t\sim Uniform(\{1,...,T\})$
    \STATE $\bar\alpha_t=\prod_{i=1}^t(1-\beta_i)$
    \STATE $x_t=\sqrt{\bar\alpha_t}x_0+\sqrt{1-\bar\alpha_t}\epsilon_1, y_t=\sqrt{\bar\alpha_t}y_0+\sqrt{1-\bar\alpha_t}\epsilon_2$ 
    \hfill\COMMENT{Eq. (\ref{e10})-(\ref{e11})}
    \STATE $\epsilon_x=model_x(x_t,y_t,t)$  
    \hfill\COMMENT{$p_x(x_{t-1}|x_t,y_t)$}
    \STATE $l_x=(\epsilon_1-\epsilon_x).square().mean()$
    \STATE update $x$ with $\nabla_{x}l_x$
\UNTIL{converged}
\RETURN $model_x$
\end{algorithmic}
\end{algorithm}
\begin{algorithm}[h]
\caption{The sampling process of paired images.}
\label{alg2}
\renewcommand{\algorithmicrequire}{\textbf{Input:}}
\renewcommand{\algorithmicensure}{\textbf{Output:}}
\begin{algorithmic}[1]
\REQUIRE noise schedule $\beta_t$, number of timesteps $T$, trained $model_x$, trained $model_y$
\ENSURE generated images $\hat x_0,\hat y_0$
    \STATE sample $x_T\sim \mathcal{N}(0,I),y_T \sim \mathcal{N}(0,I)$
\FOR{t={T,T-1,...,1}}
    \STATE $\bar\alpha_t=\prod_{i=1}^t(1-\beta_i)$
    \STATE $\epsilon_x=model_x(x_t,y_t,t)$
    \hfill\COMMENT{$p_x(x_{t-1}|x_t,y_t)$}
    \STATE $\hat x_0=\frac{x_t-\sqrt{1-\bar\alpha_t}\epsilon_x}{\sqrt{\bar\alpha_t}},x_{t-1}=\sqrt{\bar\alpha_{t-1}}\hat x_0+\sqrt{1-\bar\alpha_{t-1}}\cdot\epsilon_x$
    \hfill\COMMENT{Eq. (\ref{e7})-(\ref{e8})}
    \STATE $\epsilon_y=model_y(\hat x_0,y_t,t)$
    \hfill\COMMENT{$p_y(y_{t-1}|\hat x_0,y_t)=p_y(y_{t-1}|\frac{x_{t-1}-\sqrt{1-\bar\alpha_{t-1}}\epsilon_x}{\sqrt{\bar\alpha_{t-1}}},y_t)$}
    \STATE $\hat y_0=\frac{y_t-\sqrt{1-\bar\alpha_t}\epsilon_y}{\sqrt{\bar\alpha_t}},y_{t-1}=\sqrt{\bar\alpha_{t-1}}\hat y_0+\sqrt{1-\bar\alpha_{t-1}}\cdot\epsilon_y$
    \hfill\COMMENT{Eq. (\ref{e7})-(\ref{e8})}
\ENDFOR
\RETURN $\hat x_0,\hat y_0$
\end{algorithmic}
\end{algorithm}

\section{Experiments}
We conducted a paired image generation experiment taking the mass lesion masks ($x_0$) and the DBT slices ($y_0$) as paired images. Then, we applied the generated paired images in the supervised training of the DBT mass lesion segmentation task to verify the effectiveness of the generated images.

\subsection{Datasets and Experimental Setup}
\subsubsection{DBTMassSeg.}
This private DBT mass segmentation dataset is sourced from Suzhou Municipal Hospital and Guizhou Provincial People's Hospital\footnotemark{}. The dataset contains data of 367 patients and the data of each patient includes images from the CC and MLO perspectives and the corresponding mass lesion masks, which are manually annotated by two experienced radiologists. There are a total of 8,723 slices containing mass lesions. In the preprocessing stage, we cropped the blank parts and resized the images to 512×512. 
\footnotetext{The data used in this study were approved by the Ethics Committee of Suzhou Municipal Hospital (No. 2024320) and Guizhou Provincial People’s Hospital (No. 2024328). Written informed consent was waived as the data were fully anonymized.}

\subsubsection{Experimental Setup.}
 We linearly scaled the images to $[-1, 1]$ and set the noise schedule $\beta_t$ to constants increasing linearly from $\beta_1=10^{-4}$ to $\beta_T=0.02$~\cite{ddpm}. The maximum time step $T$ was set to 1024. 
 We used U-Net~\cite{unet} as the network architecture of the two diffusion models. We set the number of the input channels to 2 and introduced the guiding signal through concatenation in the channel dimension~\cite{sgd,ddp}. During the sampling process, we adopted the uniform step sampling method of DDIM~\cite{ddim}, with the number of sampling steps set to 256.

\subsection{Results}

\subsubsection{Focused Generation of Lesion Regions.}The unconditional paired image generation process is shown in Fig. \ref{fig3}. It can be seen that during the generation process of the lesion masks $\hat x_0$ and the DBT slices $\hat y_0$, the shape of the lesion is generated based on the lesion masks at the early stage, and the overall texture is refined in the subsequent process. It indicates that the network is more focused on the generation of lesion regions according to the diffusion guidance.
\begin{figure}
\centering
\includegraphics[width=\textwidth]{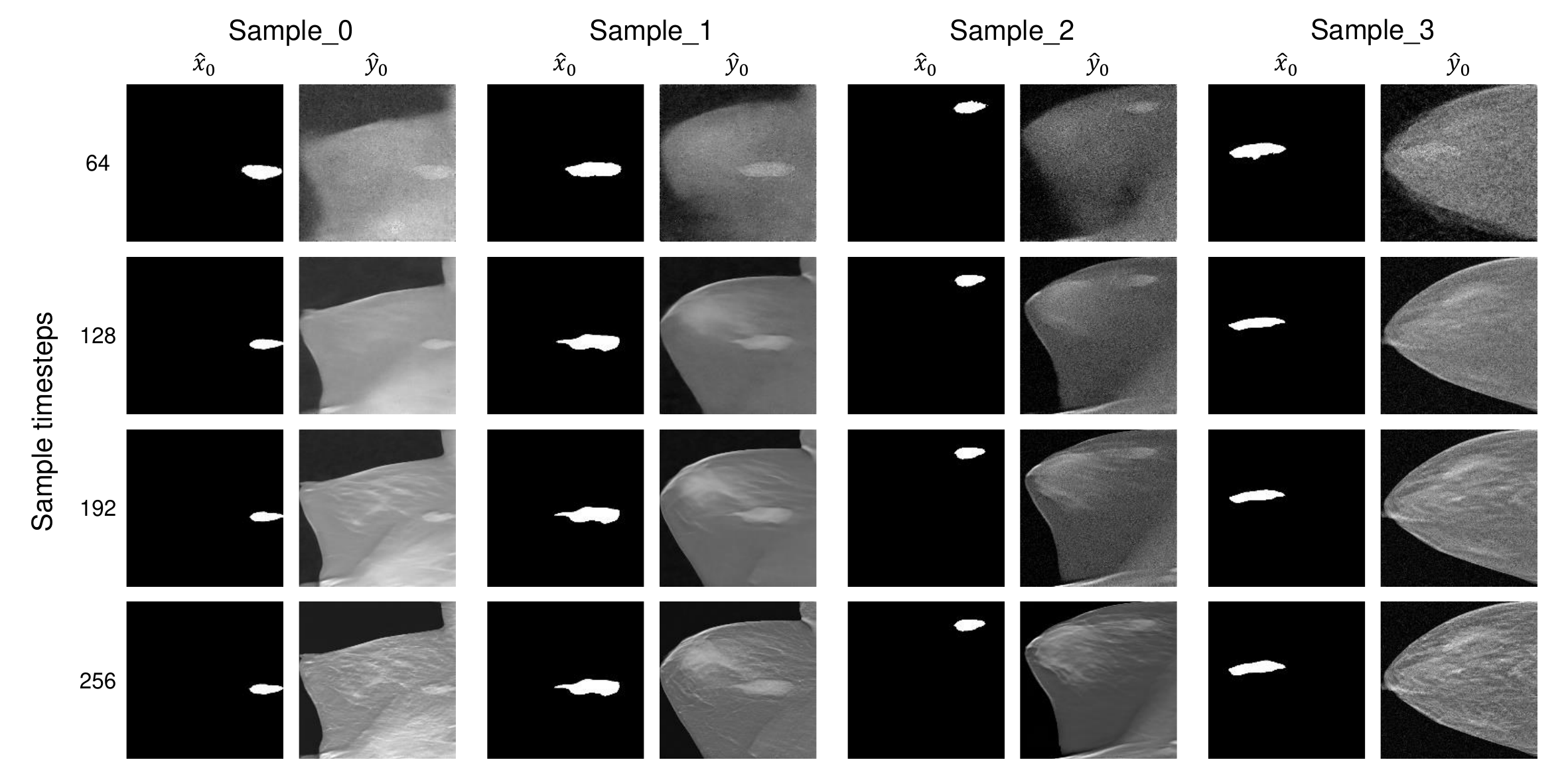}
\caption{The progressive generation process of PIG.} \label{fig3}
\end{figure}

\subsubsection{Comparative Experiments with Unconditional Diffusion Models.}
Since the proposed PIG method is a generation approach that does not require external conditions, we conducted comparative experiments with the unconditional diffusion models, including DDPM~\cite{ddpm} and DDIM~\cite{ddim}. Specifically, we set the number of sampling steps to 1024 for DDPM and 256 for DDIM, and employed each method to generate 2048 DBT slices respectively. We adopted the Frechet Inception Distance (FID)~\cite{fid} metric to evaluate the generation quality.

\begin{table}
    \centering 
    \caption{The FID metric of different methods}\label{tab1}
    \begin{tabular*}{0.6\textwidth}{@{\extracolsep{\fill}}cccc}
        \hline
        Method                  &DDPM~\cite{ddpm}   &DDIM~\cite{ddim}   &PIG\\
        \hline
        FID$\downarrow$         &31.84  &31.54  &\textbf{15.88}\\
        \hline
    \end{tabular*}
\end{table}
The experimental results presented in Table \ref{tab1} show that during the image generation process, PIG generates guiding signals simultaneously, which ensures that the process requires no additional input conditions and enhances the image generation quality.

\subsubsection{Comparative Experiments with Conditional Diffusion Models.} 
Conditional diffusion models can be guided by inputting lesion masks for image generation. Therefore, we evaluated the quality of the generated images through a lesion segmentation task. 

First, we conducted five-fold cross-validation on the DBTMassSeg dataset. Subsequently, we selected LDM~\cite{ldm}, SegGuidedDif~\cite{sgd} and ControlNet~\cite{controlnet} as comparative methods and used the existing masks in the training set as guidance to generate 2048 DBT slices respectively. Meanwhile, our proposed PIG method generated 1024, 2048 and 3072 paired images successively. Then we added the images generated by different methods to the supervised training process respectively. We used the Dice, IoU, Precision, and Recall metrics to evaluate the performance of the segmentation task.

\begin{table}
    \centering 
    \caption{The experimental results of the mass segmentation task.}\label{tab2}
    \begin{tabular*}{\textwidth}{@{\extracolsep{\fill}}ccccc}
        \hline
        Dataset&Dice(\%)$\uparrow$&IoU(\%)$\uparrow$&Precision(\%)$\uparrow$&Recall(\%)$\uparrow$\\
        \hline
        DBTMassSeg              &52.36  &41.60  &56.14  &56.81\\
        +LDM\_2048~\cite{ldm}   &56.91  &45.88  &60.92  &60.99\\
        +SegGuidedDif\_2048~\cite{sgd}      &58.30  &46.99  &60.40              &64.76\\
        +ControlNet\_2048~\cite{controlnet} &57.38  &46.09  &\underline{61.53}  &61.78\\
        +PIG\_1024          &56.52              &45.39              &59.86          &61.70\\
        +PIG\_2048          &\textbf{60.54}     &\textbf{49.10}     &\textbf{62.92} &\textbf{66.01}\\
        +PIG\_3072          &\underline{59.14}  &\underline{47.92}  &60.65          &\underline{65.66}\\
        
        \hline
    \end{tabular*}
\end{table}

The average results of the five-fold cross-validation experiment are presented in Table \ref{tab2}. Here, "+PIG\_2048" denotes the addition of 2048 annotated images generated by the PIG method to the original DBTMassSeg dataset and the notations of the other datasets follow the same principle. As shown in the results, since PIG can generate more diverse lesion masks, when generating the same number of DBT slices, the images generated by PIG can help the network learn the features of lesion areas better, thereby enhancing the network's generalization and performance. Additionally, in the experiment, all the generation methods exhibit similar trends. Specifically, the performance peaks after generating about 2048 images, and further increasing the proportion of generated images leads to performance fluctuations. For PIG, the performance dips to its lowest point after generating about 3072 images, which outperforms the highest values achieved by other methods. This indicates the high quality and stability of PIG.

\section{Conclusion}
We propose an innovative unconditional paired image generation method. First, through mathematical derivations, we prove that the generation process of paired images can be achieved by two diffusion processes guided mutually. Therefore, by training an additional diffusion guider for the conditional diffusion model, we can introduce the guiding signal into the diffusion process without external conditions. The experimental results show that our generation method can improve the generation quality and effectively alleviate the shortage of annotated data. 
\begin{credits}
\subsubsection{\ackname}
This study was partly supported by National Natural Science Foundation of China (No. 62371449 and No. 82302286), Jiangsu Provincial Key Research and Development Program Social Development Project (No. BE2022720), Guizhou Provincial Basic Research Program (Natural Science) (zk[2025]-513), and Guizhou Provincial Health Commission Science Technology Fund Project (gzwkj2024-474).
\subsubsection{\discintname}
The authors have no competing interests to declare that are relevant to the content of this article.
\end{credits}

%
%
%
%
\bibliographystyle{splncs04}
\bibliography{Paired_Image_Generation_with_Diffusion_Guided_Diffusion_Models.bbl}

\end{document}